\documentclass[11pt]{amsart}
\usepackage[english]{babel}
\newtheorem{theorem}{Theorem}

\usepackage{amssymb}
\usepackage{bm}
\usepackage[utf8]{inputenc}
\usepackage{algorithm}
\usepackage[noend]{algpseudocode}
\usepackage{filecontents}
\usepackage{graphicx}
\usepackage{subcaption}
\usepackage{tabu}
\usepackage{tabularx}
\usepackage{booktabs}
\usepackage{multirow}
\usepackage{array}
\newcolumntype{P}[1]{>{\centering\arraybackslash}p{#1}}
\newcolumntype{L}[1]{>{\raggedright\arraybackslash}p{#1}}
\newcolumntype{C}[1]{>{\centering\arraybackslash}p{#1}}
\newcolumntype{R}[1]{>{\raggedleft\arraybackslash}p{#1}}
\newcolumntype{D}{ >{\centering\arraybackslash}  }

\usepackage{tikz} 
\usetikzlibrary{positioning,fit,calc,shapes.misc}
\tikzset{block1/.style={draw=white,rounded corners,minimum height=.cm,align=center},
	line/.style={-latex}}
\tikzset{block2/.style={text width=4cm,minimum height=1cm,align=center},
	line/.style={-latex}}
\tikzset{block0/.style={draw,rounded corners,align=center, thick, minimum height=.7cm,minimum width={width("AAA")+2pt}},
	line/.style={-latex}}
\tikzset{blockcircle/.style={draw,circle,align=center, thick, minimum height=.7cm,minimum width={width("AA")+2pt}},
	line/.style={-latex}}
\tikzset{blockS/.style={draw=blue,align=center, thick, minimum height=.7cm,minimum width={width("AAA")+2pt}},
	line/.style={-latex}}
\tikzset{blockE/.style={draw=orange,align=center, thick, minimum height=.7cm,minimum width={width("AAA")+2pt}},
	line/.style={-latex}}
\usepackage{mwe} 
\tikzset{
	path image/.style={
		path picture={
			\node at (path picture bounding box.center) {
				\includegraphics[height=2cm]{hinhchen/ECGsample}};}},
}

\theoremstyle{definition}

\theoremstyle{remark}

\numberwithin{equation}{section}

\begin{document}


\title[Q-MARL for large-scale multi-agent reinforcement learning]{Q-MARL: A quantum-inspired algorithm using neural message passing for large-scale multi-agent reinforcement learning}


\author{Kha Vo}
\address{Independent Researcher}
\email{khahuras@gmail.com}
\urladdr{https://khavo.ai} 


        \author{Chin-Teng Lin}
        \address{University of Technology Sydney (UTS)}
        \email{chin-teng.lin@uts.edu.au}



\begin{abstract}
	Inspired by a graph-based technique for predicting molecular properties in quantum chemistry -- atoms' position within molecules in three-dimensional space -- we present Q-MARL, a completely decentralised learning architecture that supports very large-scale multi-agent reinforcement learning scenarios without the need for strong assumptions like common rewards or agent order. The key is to treat each agent as relative to its surrounding agents in an environment that is presumed to change dynamically. Hence, in each time step, an agent is the centre of its own neighbourhood and also a neighbour to many other agents. Each role is formulated as a sub-graph, and each sub-graph is used as a training sample. A message-passing neural network supports full-scale vertex and edge interaction within a local neighbourhood, while a parameter governing the depth of the sub-graphs eases the training burden. During testing, an agent's actions are locally ensembled across all the sub-graphs that contain it, resulting in robust decisions. Where other approaches struggle to manage 50 agents, Q-MARL can easily marshal thousands. A detailed theoretical analysis proves improvement and convergence, and simulations with the typical collaborative and competitive scenarios show dramatically faster training speeds and reduced training losses.

\end{abstract}


\maketitle

\section{Introduction}
\noindent The recent success of reinforcement learning (RL) in sophisticated cooperative/ competitive strategy games, such as DOTA \cite{DOTA2018}, StarCraft II \cite{starcraft_2019}, ATARI \cite{atari}, AlphaGo \cite{go_2016, go_2017}, and poker \cite{poker_2017, poker_2018} has raised the question of whether these algorithms can be generalised into scenarios involving a much larger number of agents. The main challenge is that the size of the joint space of actions in multi-agent RL (MARL) scenarios increases exponentially with each new agent. So, at hundreds or thousands of agents, finding optimal individual policies for each agent to complete their objectives becomes computationally infeasible. Moreover, training any single agent without taking the behaviour of other agents into account would violate the compulsory Markov property assumption of the environment's stationarity \cite{survey2, survey5, survey10}. From this violation, a cascade of other problems follows, including suboptimal credit assignments \cite{survey32}, imbalanced exploitation/exploration \cite{survey8}, combinatorial explosions \cite{survey347}, and policy overfitting \cite{survey35}. Hence, despite the aid of advanced machine learning techniques like deep neural networks \cite{survey30}, designing algorithms to solve the non-stationarity of MARL remains the pivotal problem in RL \cite{survey29}. \\

\noindent In MARL, non-stationarity problems are not straightforward to address, especially once the number of agents exceeds just a few dozens. For example, combinatorial explosion means that one cannot naively view the joint state and action space of all agents as a single space and then implement single-agent algorithms. This also precludes the use of central controller-like algorithms that distribute actions to agents \cite{Lowe17, Gupta17, Omidshafiei17, Foerster16}. Some researchers have attempted to override the Markov stationarity property by making a few assumptions about the environment, such as a common reward system \cite{Wang02, Arslan17}. However, with too many agents, the state and action spaces become too large to learn trustworthy policies and can even fail to converge to a local optimum \cite{Lauer00, Littman01}.\\

  \noindent In competitive and collaborative MARL scenarios, the actions of agents are correlated. The degree of influence these correlations have is governed by the \textit{neighbourhood} information each agent contains. For instance, in a simulated soccer game, the action of a striker with the ball is strongly affected by the surrounding agents, i.e., the opponent's defenders and the teammate's midfielders, but weakly affected by the agents farther away, i.e., the teammate's goalkeeper and defenders. Similarly, in a cyber-security scenario, connected computers form a graph network, where a malicious attack must find the intermediate agents on its way to the final computer target. These examples demonstrate that focusing on appropriate information is more efficient than using all  information that can possibly be retrieved from the environment.\\

\subsection{Related Works} \label{sec:related}
\subsubsection{Issues in Centralised Learning}

\noindent The collaborative and competitive nature of MARL leads to a handful of difficult challenges, the most notable of which is the possibility of an extremely large joint state/action space. In fact, early research \cite{Lauer00, Littman01} on value-based approaches found this problem to be intractable at a certain number of agents. In looking for solutions, some researchers turned to assumptions about the environment. For example, Wang and Sandholm \cite{Wang02} assumed a common reward system for all agents so the value function could be estimated locally (see also \cite{Arslan17}). This simple assumption makes training trivial; however, it also violates the stationarity characteristic required for RL \cite{survey2, survey5, survey10}.\\

\noindent Deep learning has also been explored for answers by embedding neural networks into the policy learning process as a viable function approximator \cite{Lowe17, Gupta17, Omidshafiei17, Foerster16}. The general approach here is to implement a common reward system \cite{Omidshafiei17, Gupta17} to control each agent's learning with a central manager \cite{Foerster16, Lowe17}. However, combinatorial explosion prevents all of these approaches from scaling to large environments. Hence, to reduce the joint state/action spaces and speed up the learning process, some frameworks strike a balance between full-scale and local training by allowing communication between agents. A full-scale training using the attention mechanism for collaborative tasks with two policy networks in an end-to-end architecture was proposed and named ATOC \cite{Jiang18}. This work is an extension of the multi-agent deep deterministic policy gradient (MADDPG) \cite{Lowe17}. However, in MADDPG, the authors also assumed the existence of fully connected agents in the communication network. TarMAC \cite{Das18} provides an architecture for targeted continuous communication via soft attention that allows agents to select other appropriate agents to communicate with at each time step. Both methods show that incorporating the attention mechanism improves performance with complete centralised small-scale scenarios but, again, large-scale scenarios are still beyond computational limits. 

\subsubsection{Agent Communication}

 One typical approach that aims to achieve a balance between full-scale training and local training to accelerate the process is to allow partial communication between agents. Communication in RL is an emerging field of research for collaborative tasks in a partially observable environment. Reinforced inter-agent learning (RIAL) and differentiable inter-agent learning (DIAL) were proposed in \cite{162} to discover the communication protocols between agents. These methods employ a neural network to determine each agent's Q values, as well as the communication message to all other agents in the next time step. As a result these methods work by centralised learning and decentralised execution. The centralised learning is not feasible in large-scale scenarios, i.e, the authors of RIAL/DIAL only experimented on small scenarios with up to 4 agents. Memory-driven multi-agent deep deterministic policy gradient (MD-MADDPG) \cite{165} was built on top of the classic MADDPG \cite{Lowe17} by adding a shared memory as a communication protocol across agents. This shared memory is processed by each agent before taking an action, and is dependent on each agent's private observation. Similar to RIAL/DIAL, MD-MADDPG also lacks a mechanism to deal with large-scale environment, because the training flow will explode if the number of agents exceeds a few dozens. The authors of MD-MADDPG only conducted experiments in toy problems with up to 2 agents. Similarly, Dropout-MADDPG \cite{177} was inspired by the dropout technique in supervised learning \cite{dropout} and MADDPG, but also suffered from the \textit{curse of agent count}. Although the dropout mechanism implemented on communication messages had some positive effects such as improved training speed and steady-state performance, it cannot compensate the great drawback of the infeasibility in large scenarios, as all experiments only involved less than 10 agents. CommNet  \cite{Sukhbaatar16} provides continuous communication through a simple broadcasting channel. In this strategy, an averaged message from all agents in one network layer is used as input for the next layer. However, although CommNet has been shown to work well in collaborative tasks, it does not consider competitive scenarios or scenarios with heterogeneous agents. Designed as a robust approach to combat and capture games, and StarCraft in particular, BiCnet \cite{Peng17} is a vectorized extension of the actor-critic approach that employs a bi-directional recurrent communication framework. However, since full-scale training is required, feasible scenarios are limited to a maximum of 50 agents.\\

\subsubsection{Approaches for big-scale scenarios}

Of the impact-based strategies to tackle large scenarios, mean-field MARL \cite{Yang18b} approximates the mean effect of neighbouring agents. This method has two variant, one policy-based (mean-field actor-critic) and the other value-based (mean-field Q-learning). Although applicable to large-scale scenarios with homogeneous agents, this approach only considers the influence of neighbours at a very broad level; the nuanced information is lost. Therefore, in simple scenarios with a diverse range of heterogeneous agent types, such as speaker-listener \cite{Lowe17}, the mean-field approach has difficulty representing the dominant influence of the speaker. \\

\noindent Jaques et al. \cite{Jaques18} is the other impact-based framework. This solution involves thoroughly assessing social influence across all agents by comprehensively simulating the alternative actions each agent could have taken. Actions with the greatest influence over other agents are detected and encouraged later in the process. This strategy is good for predicting degrees of influence across diverse agent types, but it does not establish a strong connection between the causal prediction and the real target. In other words, the causal analysis is unsupervised and can be prone to failure in scenarios where the probability of successful actions during training are low. \\  

\noindent More recently, there have been some attempts to integrate the concept of local neighbourhoods into MARL through attention \cite{Jiang20, Agarwal19} and relevance graph embedding \cite{Malysheva18}. Since these approaches are typical to be compared with our proposed method with regards to the relatively large number of agents, we will go into more details with some experimental comparison in Section \ref{sec:compare}. \\

\subsection{Inspiration from Quantum Chemistry}
Predicting molecular properties is regarded as one of the most challenging problems in quantum chemistry \cite{nmp}, in part because the atomic symmetry of molecules demands a graph-based mechanism that can deal with the isomorphic invariance of an atom's position in three-dimensional space. Example problems have even been used as the basis of world-class machine learning competitions, with thousands of entrants across the globe \cite{CHAMPSKaggle}. As a team who finished in the gold-medal zone of one of these competitions in 2019 (rank 10/2749) \cite{CHAMPSKaggle}, we were intrigued to find some interesting parallels between quantum chemistry and MARL as follows:

\begin{itemize}
	\item Consider that the atoms in a molecule in quantum chemistry are equivalent to the agents in an interactive game in MARL. In both scenarios, the relative pair-wise distances between entities are crucial.
	\item Rotational invariance and graph isomorphism are equivalent and are key, in that randomly rotating a molecule is equivalent to changing the coordinate origin in a MARL game, and placing entities in an equivalent symmetric way should give the same outcome.
	\item The number of entities is not fixed, rather it differs given the context.
	\item The influence of one entity over another is governed by proximity.
	
\end{itemize}
These similarities prompted us to wonder whether state-of-the-art methods in quantum chemistry might be state-of-the-art for MARL. However, transposing a working concept from quantum chemistry to MARL poses multiple technical challenges. The first being how to reformulate a supervised learning problem with a fixed dataset in quantum chemistry into an RL problem to be solved by self-playing with multiple agents given infinite simulating data.\\

%

\begin{figure}
	\includegraphics[width=\linewidth]{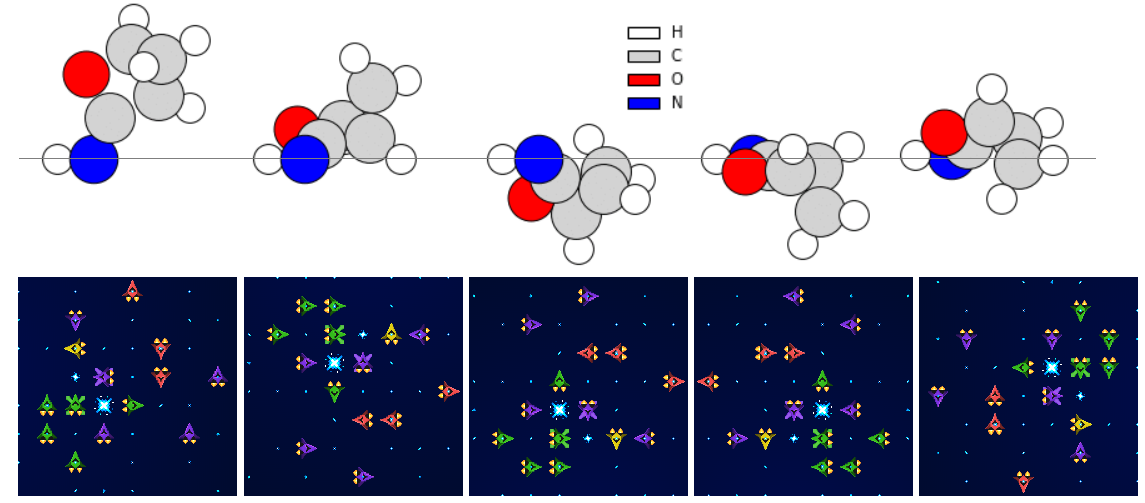}
	\caption{Inspiration for MARL from quantum chemistry. Top row: illustrations of five differently placed counterparts of a same molecule $C_4H_5ON$. The \textit{scalar coupling constant} \cite{CHAMPSKaggle} between any pair of atoms in the molecule is invariant with respect to the coordination of the molecule. Bottom row: illustration of five differently placed scenes of the same relative structure in a 4-player Halite game \cite{Halite, Halite4}. The optimal action of each agent should be invariant with respect to the scene's rotation/symmetry.}
	\label{fig:QC}
\end{figure}

\noindent In this paper, we propose a framework called Q-MARL (Quantum-MARL), to fully decentralise learning, inspired by state-of-the-art works from quantum chemistry. Our graph-based approach is to address large-scale MARL in joint collaborative-competitive scenarios. The model is capable of capturing nuanced influence across agents within a neighbourhood by efficiently decomposing the large original environment into sub-graphs. The key characteristic of our approach is the invariance of the model with respect to the rotation or isomorphism of graphs representing neighborhoods of agents. At each time step, each agent is able to observe only its own information, i.e., state and reward, and the information of agents in its neighbourhood. Stated differently, the neighbouring information forms a time-varying graph, where each vertex is an agent and an edge between two agents indicates a temporary neighbouring bond between them. The agents do not necessarily form a full graph in which any vertex can be accessed by any vertex. The time-varying property of the graph guarantees convergence to global optimality while substantially reducing the size of the problem. Our graph-based approach addresses this problem by not discriminating each agent's identity. A message-passing mechanism is used as the main backbone of the graph network, and all vertices and edges can fully interact with each other in a sub-graph. By accumulating the gradient of each specific agent in all sub-graphs that it belongs to, training becomes more stable and converges faster. The graph network is coupled with a $\lambda-step$ actor-critic RL algorithm \cite{AC} and is deployed in a decentralized fashion, i.e., at test time, each agent acquires its own information and the neighbourhood information, forms a graph, and executes an action from the trained model without the existence of a central manager. \\

\subsection{Research Gaps and Contributions}

Given the literature review in Section \ref{sec:related}, the following research gaps that need to be fulfilled are summarized as follows.
\begin{itemize}
	\item Previous works that proposed decentralized training schemes in big-scale scenarios only use shared model parameters with respect to each neighborhood, thus actions of agents within a neighborhood are independent. This is not desired. There lacks a centralized training approach within each decentralized neighborhood.
	\item To the best of our knowledge, there were no works dealing with the graph's rotation and isomorphism to tackle overestimation and improve generalization in MARL. We were inspired by this idea from quantum chemistry where molecules' rotational invariance and graph isomorphism play an important role in model design. The advent of this innovative idea also tailors with the research gap made in the previous point:  a dynamic graph representing a network of agents with varying size/topology/positions is the most generalized and can be viewed as a centralized fashion within each neighborhood.
	\item There lacks an approach to fully centralize each local neighborhood sample. Suppose after decentralizing the full scenario into overlapped neighborhoods, the action of each agent is desired to be influenced by the actions all other agents it communicates to. This is not the case in most works in literature.
\end{itemize}

The contributions of our work can be generally summarized as follows.
\begin{itemize}
	\item Formulate a general MARL scenario as a graph-based problem, where homogeneous agents' policies can be employed in a decentralized fashion. Convergence to global optimality is guaranteed and proved by assuming a time-varying graph transition.
	\item  Propose a graph-based model based on a message-passing neural network mechanism, where all vertices and edges can interact with each other. The rotational invariance of the input vertices is a crucial factor for the convergence analysis mentioned in the first point.
	\item Implement our proposed method in typical multi-agent scenarios, with comparison to recent state-of-the-art algorithms in the literature. Our results suggest that given an existing algorithm, the convergence speed is drastically enhanced by using the sub-graph embedding, and the decentralized framework is preserved with promising global reward performance. Our method is therefore well tailored to large-scale competitive and collaborative multi-agent tasks. 
\end{itemize}

\noindent The rest of this paper is organised as follows. Section \ref{sec:theory} presents theoretical aspects of the work, including the convergence analysis of policy gradient in MARL (Section \ref{sec:theory1}), graph-based formulation of the problems (Section \ref{sec:theory2}), and neural message passing architecture (Section \ref{sec:architecture}). The experiment scenarios and results are presented in Section \ref{sec:scenarios} and Section \ref{sec:mainresult} respectively, followed by the comparison with other studies in Section \ref{sec:compare}. Finally, Section \ref{sec:conclusion} draws key brief conclusion points. \\

\section{Graph-Based Multi-Agent Reinforcement Learning}\label{sec:theory}
\subsection{Background and Problem Setting}\label{sec:theory1}
   Consider an environment with $N$ agents, each having a finite state space $\mathcal S^i$ and a finite action space $\mathcal A^{i}$. The joint state and action spaces of all agents are denoted as $\mathcal S = \Pi_{i}\mathcal S^{i}$ and $\mathcal A = \Pi_{i}\mathcal A^{i}$, respectively. In each time step $t$, agent $i$ is in an instantaneous state $S_t^{i}$, performs action $A^{i}_t$ drawn from its current policy $\pi^{i}(a^i|s)$ which is parametrised by $\bm\theta^{i}$, then transitions to a new state $S_{t+1}^i$ and receives a reward $R^{i}_{t+1}$.  From a global perspective, the joint state and joint action of all agents at time step $t$, are denoted as $S_t = (S_t^{1}, \dots, S_t^{N})$ and $A_t = (A_t^{1}, \dots, A_t^{N})$, respectively. \\
   
   \noindent The parametric policy for each agent is formulated as
   \begin{equation}
   \pi^i(a^i|s) = \mathbb P\big\{ A_t^i=a^i \; \big | \; S_t=s, \; \bm\theta^i\big\}
   \end{equation}
   This policy indicates the probability of agent $i$ taking action $a^i$ given the joint state $s$ and the agent's policy parameters $\bm\theta^i$ at time step $t$.  This formulation expresses the multi-agent scenario as a single-agent scenario, with the joint state and action at each time step as the state of action of a single agent. The sample of joint return at time step $t$ can be formulated simply as the averaged return of all agents, i.e., $R_t = \frac{1}{N}\sum_{i}R_t^{i}$.  The optimal joint policy $\tilde{\bm\theta} = (\tilde{\bm\theta}^1, \dots, \tilde{\bm\theta}^N)$ is obtained by maximizing the expectation of the long-term discount reward 
   \begin{equation}
    \tilde{\bm\theta} = \text{maximize}_{\bm\theta} \lambda(\bm\theta) \doteq \mathbb E[G_t | \bm\theta ] 
   \end{equation}
   where $G_t = \sum_{t=0}^{\infty} \gamma^{t}R_{t+1}$ denotes the instantaneous long-term reward and $\gamma \in (0,1]$ is a discount factor influencing the impact of near-future actions. We hereby present the extended policy gradient in a multi-agent environment. 

\subsection{Theoretical Foundations}\label{sec:theory2}
\begin{theorem}\label{theorem_1} The gradient of the global reward 
$\lambda(\bm\theta)$ with respect to any local policy parameter $\bm\theta^{i}$ of agent $i$ can be derived as
\begin{equation}
\nabla_{\bm\theta^{i}}\lambda(\bm\theta) = \mathbb E_{\pi } \Big[ G_t  \nabla_{\bm\theta^{i}}  \ln \pi^{i} (A_t^{i} | S_t )  \Big],
\end{equation}
where $\mathbb E_\pi$ denotes the expectation of the random variables inside the brackets following the joint policy $\pi$.
\end{theorem}

\begin{proof}
	Conventional policy gradient (PG) theorem \cite{pg} states that the gradient of the long-term discounted return with respect to the joint policy $\bm\theta$ is
	\begin{equation}\label{Eq:PG}
	\begin{aligned}
	\nabla \lambda(\bm\theta) & = \sum_{s\in \mathcal S}d_\pi(s)\sum_{a\in \mathcal A} q_\pi(s, a)\nabla \pi(a|s) \\ &  = \mathbb E_{\pi} \bigg[ G_t \frac{\nabla\pi(A_t|S_t)}{\pi(A_t|S_t)} \bigg].
		\end{aligned}
		\end{equation}
	 The first equality is derived as the original result of PG \cite{pg}, while the second equality holds by replacing the variables $s$ and $a$ with stochastic samples $S_t$, $A_t$ and obtaining the expectation, as explained in \cite{RL_book}. Since by definition $\pi(A_t|S_t) = \Pi_i \pi^{i}(A_t^{i} | S_t)$, then from (\ref{Eq:PG}), the proof is straightforward by taking the gradient with respect to only local parameters $\bm\theta^i$, given the fact that all individual policies are conditionally independent, i.e., $\nabla_{\bm\theta i}\ln \pi^j (A_t^j|S_t) = 0 \;\; \forall j \neq i$, and $ \frac{\nabla_{\bm x}f(\bm x)}{f(\bm x)} = \nabla_{\bm x}\ln f(\bm x) \; \forall f$.
\end{proof}

\noindent Here, we introduce two widely used terms in RL - the \textit{state-value} function $v_{\pi}(s)  \doteq \mathbb E_{\pi}[G_t |S_t=s]$  and the \textit{action-value} function $q_{\pi}(s,a) \doteq \mathbb E_{\pi}[G_t  |  S_t=s,A_t=a]$. The state-value function expresses the expectation of long-term reward in a specific state $s$  following policy $\pi$,  while the action-value function expresses the expected long-term reward when starting in state $s$, performing action $a$, and following $\pi$. 
The policy gradient theorem of graph-based MARL suggests that the gradient of the global reward with respect to the local policy parameter $\bm\theta^{i}$ of agent $i$ is independent of the other agents' policies and requires only the unbiased estimate $G_t$ of $q_{\pi}$. However, in our decentralised graph-based MARL we do not have access to $G_t$ during training. Therefore, each agent is coupled with a local estimate $\hat q^{i}(s, a)$, which is parametrized by $\bm w^{i}$. Further details are provided in Section \ref{sec:architecture}.

\begin{theorem}\label{Theorem:2}
An agent's policy $\bm \theta ^{i}$ can be updated locally via the following procedure
\begin{equation}
\bm\theta^{i} \leftarrow \bm\theta^{i} + \alpha\delta^{i}\nabla_{\bm\theta^{i}}\ln\pi^{i}(A^{i}_t | S_t),
\end{equation}
where $\delta^{i} =\hat q^{i}(S_t, A_t) - v(S_t, A_t^{(-i)})$ and $A_t^{(-i)}$ denote the actions of all agents who are independent of agent $i$ at time step $t$.
\end{theorem}
\begin{proof}
From Theorem \ref{theorem_1}, it is evident that the sample $G_t \nabla_{\bm\theta^i}\ln \pi^i(A^i_t|S_t)$ serves as a stochastic estimate following the Monte Carlo gradient ascent algorithm \cite{MC}. With a proper learning rate $\alpha$, the update procedure for $\bm\theta^i$ is
\begin{equation}
\bm\theta^{i} \leftarrow \bm\theta^{i} + \alpha G_t\nabla_{\bm\theta^{i}}\ln\pi^{i}(A^{i}_t | S_t).
\end{equation}
From a general perspective, $G_t$ is $\lambda$ steps in the future and difficult to retrieve. Thus, this value can be substituted with a semi-gradient target $G_t \approx R_{t+1} + \gamma  v(S_{t+1})$ \cite{RL_book}. However, $R_{t+1}$ is still inaccessible locally, and the true global state-value function $v$ is unknown. Thus, an alternative unbiased estimate of $G_t$ must be obtained. To this end, during training, each local agent is coupled with its own action-value network that serves as a critic, denoted by $\hat q ^ i(s, a)$ and parametrized by $\bm w^i$. According to the definition of the action-value function, $\hat q^i$ is a direct unbiased estimate of $G_t$. Moreover, adding a baseline offset quantity $v(S_t, A_t^{(-i)})$ to $G_t$ preserves the update procedure and stabilizes the training process by reducing the target variance since from (\ref{Eq:PG}), we have $\sum_{a^i} v(s, a^{(-i)}) \nabla_{\bm\theta^i} \pi^i (a^i | s) = v(s, a^{(-i)})  \nabla_{\bm\theta^i}\sum_{a^i}\pi^i(a^i|s) =  0$ because $\sum_{a^i} \pi^i(a^i|s) = 1$. 
\end{proof}
\noindent The baseline can be computed by sweeping through all possible actions $a^i$ combined with the instantaneous samples of other agents $A_t^j, \; j\neq i$ as
$$
v^{i}(S_t, {A_t^{(-i)}} ) = \sum_{a^{i}} \pi^{i}(a^{i}|S_t) \; \cdot \;  \hat q^{i}\Big(  S_t, \big(A_t^{(1)}, \dots, A_t^{(i-1)}, a^{i}, A_t^{(i+1)}, \dots, A_t^{(N)}  \big) \Big).
$$

\noindent\textbf{Graph-based Formulation}\\

\noindent According to Theorem \ref{Theorem:2}, the global state $S_t$ must be accessed in each time step to train every individual agent. Yet , this is a difficult task in large-scale MARL scenarios because, as discussed, an agent cannot retrieve the global state during testing due to specific environmental specifications. Decomposing the full environment into smaller sub-environments, each one specific to an individual agent, offers multiple benefits. First, agents require far less but suitable training data in each time step which accelerates the training process, plus the acquired data for each agent is concentrated on the local environment, which means irrelevant or spurious interactions with other uncorrelated agents can be discarded leaving only quality data. Second, each agent's decision can be determined by an ensemble of its actions in $N$ sub-graphs $\mathcal G^i$ during each time step, i.e., from its actions in all the neighbourhoods it belongs to, which makes for better outcome and also helps to stabilise the training process. \\

 \noindent We model the full environment as an undirected graph $\mathcal G (V, E)$ consisting of a set of vertices $V$ and a set of edges $E$. Each vertex in $V$ represents one agent in the environment, and each edge in $E$ is a bidirectional connection between two agents. For each agent, there is one sub-graph $\mathcal G^i$ based on the perspective of agent $i$. All other agents in the neighbourhood of agent $i$ are present in $\mathcal G^i$. As a result, multiple sub-graphs can contain the same agent.\\

\noindent Some modifications to the vanilla RL algorithm have to be made. Without loss of generality, all decomposed sub-graphs have identical roles in learning from the global perspective. In other words, the sub-graphs of one time step in MARL can be thought of in the same way as one-batch data samples in supervised learning. This treatment allows a single model to be built for the whole environment. The model takes each sub-graph as input and stochastically performs back-propagation on each agent in the sub-graph. The model has two major components: the main policy $\pi(a|s)$ and the action-valued critic $\hat q(s, a)$. With a slight abuse of notation,  $\pi$ and $\hat q$ operate only on sub-graphs; hence, the action variable $a$ and state variable $s$ represent the joint action and joint state of all agents in a sub-graph. The output of the model is agent-wise, which means each agent $j$ in sub-graph $\mathcal G^i$ is assigned an action according to a probability policy $\pi (a^i_j| s^i)$. The pseudo-code of our algorithm presented in Algorithm \ref{Alg1}, where we have replaced $s^i$ with $\mathcal G^i$ for notational convenience. The graph decomposition process is illustated in Figure \ref{fig:maingraph}.\\

\begin{algorithm}
	\caption{Graph-Based MARL}\label{Alg1}
	\begin{algorithmic}[1]
		\State Initialize: Graph policy model $\pi(a|s)$ with parameters $\bm\theta$ 
		\State Initialize: Action-value function $\hat q(s, a)$ with parameters $\bm w$
		\State Input: learning rates $\alpha_{\bm\theta}, \alpha_{\bm w}$
		\For {each episode}
		\For {each time step}
		\If {episode end criteria met}: break
		\EndIf
		\State $\mathcal G^i \leftarrow$ decompose$(S)$
		\State  Sample $A^{i}_j \sim \pi(\mathcal G^i) $ for all relevant $j$ in each $\mathcal G_i$
		\State Ensemble action $A^i = \sum_j A^{i}_j$ for each i
		\State Perform actions $A = [A^{i}, \dots, A^{N} ]$ to the environment
		\State Observe $\mathcal G' = [\mathcal G'^1, \dots, \mathcal G'^N], R=[R^{1}, \dots, R^{N}]$
		\For {each $i$}
		\For {each relevant $j$ in $\mathcal G^i$}
		\State $v\leftarrow \sum_{j} \pi(a_{j}|\mathcal G^i) \cdot  \hat q\Big(  \mathcal G^i, \big[A^{i}_1, \dots, a^{i}_j, \dots, A^{i}_{\text{end}}  \big] \Big)$
		\State $v'\leftarrow \sum_{j} \pi(a_{j}|\mathcal G'^i) \cdot  \hat q\Big(  \mathcal G'^i, \big[A^{i}_1, \dots, a^{i}_j, \dots, A^{i}_{\text{end}}  \big] \Big)$
		\State $\delta \leftarrow R^{i} + \gamma v' - v$
		\State $\bm w \leftarrow \bm w +\alpha_{\bm w}\delta \nabla_{\bm w}\hat q(\mathcal G^i, A^i)$
		\State $\bm\theta \leftarrow \bm\theta + \alpha_{\bm\theta} \delta \nabla_{\bm\theta}\ln \pi( A^i_j | \mathcal G^i )$
		\EndFor
		\EndFor
		\EndFor
		\EndFor
	\end{algorithmic}
\end{algorithm}

\begin{theorem}\label{Theorem:3} 
	The ensemble action of each agent $i$ aggregated from all relevant sub-graphs that contain it, denoted as $A^i$ in Algorithm \ref{Alg1}, is statistically improved from any single action $A^i_j$ from any sub-graph that contains agent $i$.
	
\end{theorem}
\begin{proof}
\noindent Suppose for a specific agent $i$, we draw each scenario circumstance $(y,\bm x)$ from the probability distribution $P$ to form a sub-graph $\mathcal G$ containing agent $i$, where $y$ is the optimal action for agent $i$ and $\bm x$ is the input feature vector of agent $i$ in $\mathcal G$. Let $\pi(\bm x, \mathcal G)$ be the learner's output action that is performed on $\mathcal G$, then subsequently, the ensemble action averaged on all possibilities of $\mathcal G$ can be defined as follows:

\begin{equation}
\pi_{\text{ensemble}}(\bm x, P) = \mathbb E[\pi(\bm x, \mathcal G)].
\end{equation}

\noindent We denote the random variables of $\bm x, y$ as $\bm X, Y$, which are independent of $\mathcal G$. The averaged error $e$ on a weak single action $\pi(\bm x,\mathcal G)$ over all possibilities of $\mathcal G$ can be calculated by 

\begin{equation}
\begin{aligned}
e = \; \; & \mathbb E_{\mathcal G}\mathbb E_{\bm X, Y}[Y-\pi(\bm X, \mathcal G)]^2\\ = \; \; & \mathbb E_{\mathcal G}\mathbb E_{\bm X, Y}[Y^2] - 2\mathbb E_{\bm X, Y}[Y]\mathbb E_{\mathcal G}[\pi(\bm X, \mathcal G)] + \mathbb E_{\bm X, Y} \mathbb E_{\mathcal G}[\pi(\bm X, \mathcal G)]^2 \\  = \;\; & \mathbb E_{\bm X, Y}[Y^2] - 2\mathbb E_{\bm X, Y}[Y\pi_{\text{ensemble}}] + \mathbb E_{\bm X, Y} \mathbb E_{\mathcal G}[\pi(\bm X, \mathcal G)]^2.
\end{aligned}
\end{equation}

\noindent On the other hand, the error produced by $\pi_{\text{ensemble}}$ is computed as

\begin{equation}
\begin{aligned}
e_{\text{ensemble}} = \; \; & \mathbb E_{\bm X, Y}[Y-\pi_{\text{ensemble}}(\bm X, P)]^2\\ = \; \; & \mathbb E_{\bm X, Y}[Y^2] - 2\mathbb E_{\bm X, Y}[Y\pi_{\text{ensemble}}] + \mathbb E_{\bm X, Y} [ \mathbb E_{\mathcal G}\pi(\bm X, \mathcal G)]^2.
\end{aligned}
\end{equation}

 \noindent Since $\mathbb EZ^2 \geq \mathbb E^2 Z \; \forall  Z$ we yield
 
 \begin{equation}
 e\geq e_{\text{ensemble}}
 \end{equation}

 \noindent The gap of $e$ and $e_{\text{ensemble}}$ is dependent on the variance of $\pi_{\text{ensemble}} (\bm X,P)$ as
  
 \begin{equation}
 \begin{aligned}
 e - e_{\text{ensemble}} \geq \; \; & \mathbb E_{\bm X, Y}  \mathbb E_{\mathcal G}[\pi(\bm X, \mathcal G)^2] - \mathbb E_{\bm X, Y}[\mathbb E^2_{\mathcal G}\pi(\bm X, \mathcal G)] \\ = \;\; & \mathbb E_{\bm X, Y}[\text{Var}_{\mathcal G} \pi(\bm X, \mathcal G)].
 \end{aligned}
 \end{equation}
\end{proof}

\noindent Thus, we expect to obtain an improved action by ensembling single actions from the population of sub-graphs $\mathcal G$ drawn from $P$ that contains a specific agent $i$. This is achieved by performing the inference of all agents in each sub-graph after the environment decomposition step, then aggregating the actions of agent $i$ from the sub-graphs it belongs to. 

\begin{figure}[]
	\begin{subfigure}{.9\textwidth}
		\centering
		\includegraphics[width=1\linewidth]{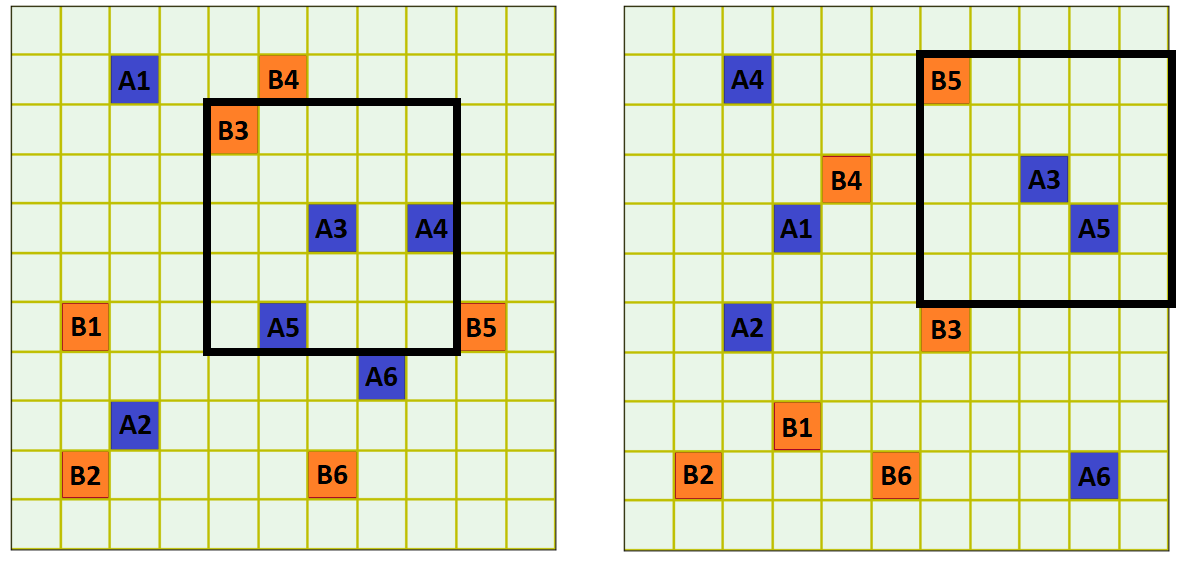}  
		\caption{Scenario Displacement}
		\label{fig:maingrapha}
	\end{subfigure}
	\begin{subfigure}{.9\textwidth}
		\centering
		\includegraphics[width=1\linewidth]{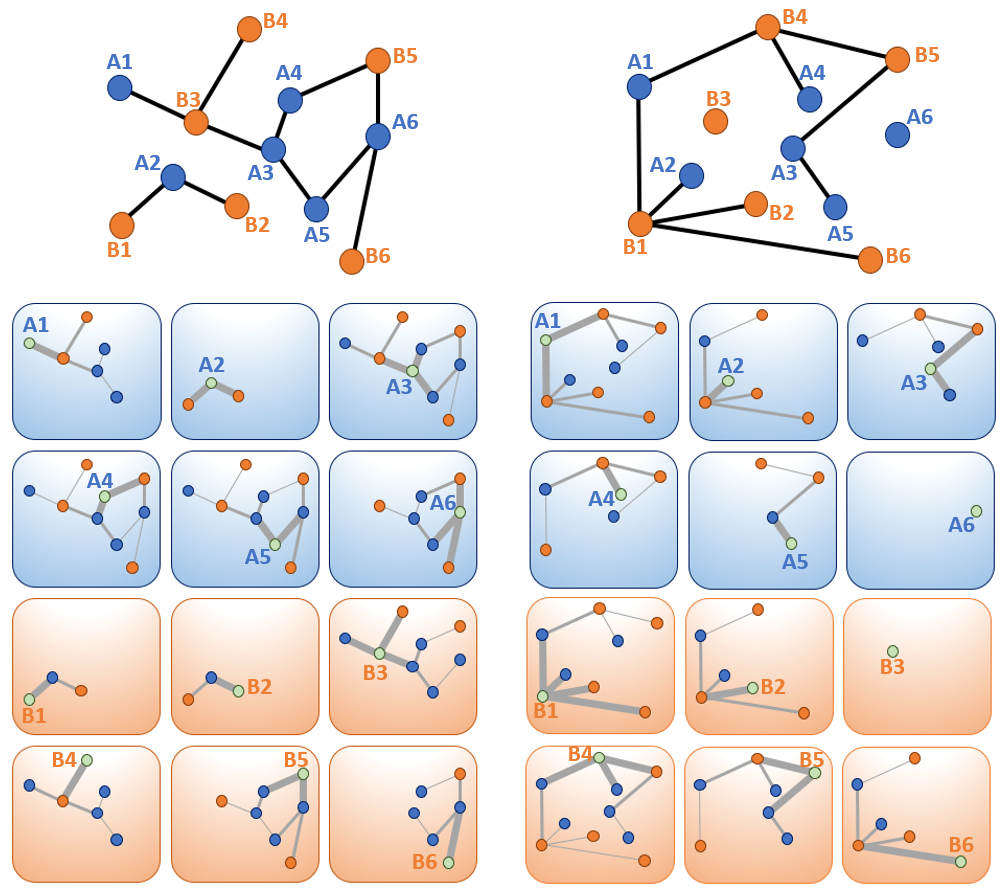}
		\caption{Graph Formulation}
		\label{fig:maingraphb}
	\end{subfigure}
	\caption{An example of how the graph decomposition process differs for two different time steps. The blue and orange vertices represent two homogeneous groups of agents, i.e., two teams. There are 12 sub-graphs for each time step, each formed from the perspective of a different individual agent as indicated by the green vertex, and extending to its 3rd-degree neighbours. This depth of degree, and therefore the complexity of the model, is controlled by a hyperparameter.}
	\label{fig:maingraph}
\end{figure}

\subsection{Neural Message Passing}\label{sec:architecture}
Our graph-based model architecture is inspired by the neural message-passing (NMP) networks of quantum chemistry \cite{nmp}. We argue that a large-scale MARL environment is similar to atomistic space in quantum chemistry, where each agent in a MARL scenario is equivalent to an atom in a molecule, and the neighbouring information across agents is equivalent to the inter-atomistic interactions between atoms. The architecture of our proposed network is illustrated in Figure \ref{fig:architecture}. The feature vector for the state of each agent $i$ in the graph is denoted as $\bm s_i$, and the edge, i.e., the neighbouring information between two agents $i, j$ is denoted as $\bm z_{ij}$. The hidden features of agent $i$ at layer $l$ are dispersed to its neighbours $j \in \mathcal N(i)$ by a vertex update block $\bm{\mathcal V}$ as
\begin{equation}
\bm s_i^{l+1} = \bm{\mathcal{V}} (\bm s_i^l, \bm z_{ij}^l).
\end{equation}
The edge update process for $\bm z_{ij}^l$ is dependent on itself and the resulting updated vertices at its two ends $\bm s_i^{l+1}, \bm s_j^{l+1}$ as
\begin{equation}
\bm z_{ij}^{l+1} = \bm{\mathcal{E}} (\bm s_i^{l+1}, \bm s_j^{l+1}, \bm z_{ij}^l).
\end{equation}

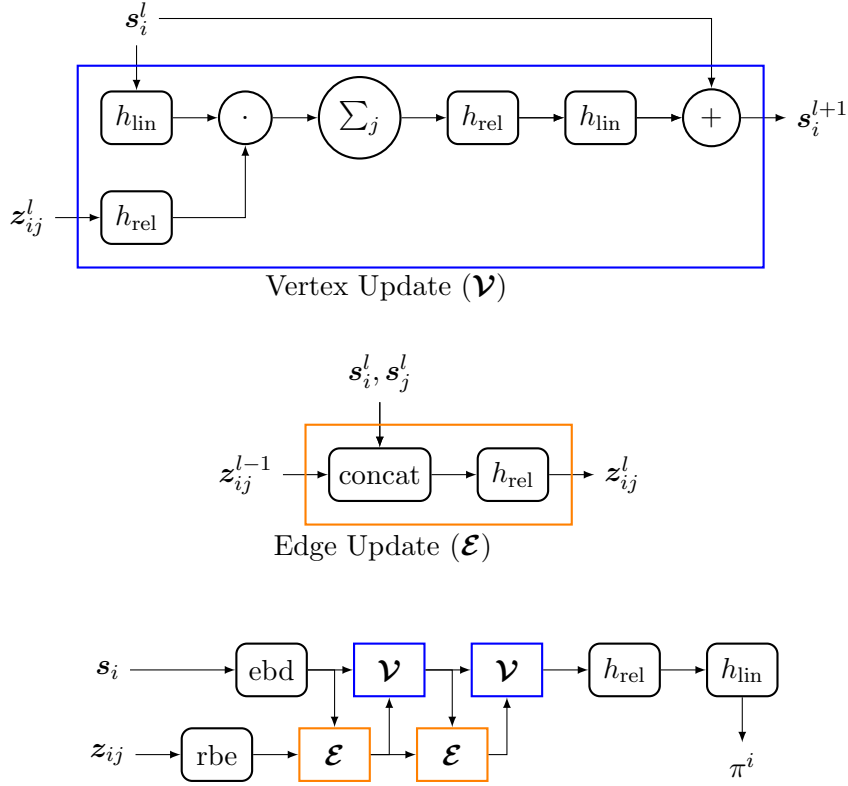
\begin{figure}[t]
	\centering
		\begin{subfigure}{1\textwidth}
		\centering
		\begin{tikzpicture}
		\node[block0] (fc11) {$h_{\text{lin}}$}; 
		\node[block1, above = .6 of fc11] (x) {$\bm s_i^l$};
		\node[block0,  below=.6  of fc11] (fc21) {$h_{\text{rel}}$};
				\node[block1, left=.6 of fc21] (z) {$\bm z_{ij}^l$};
		\node[blockcircle,  right= .6 of fc11] (dot) {.}; 
		\node[blockcircle,  right= .6 of dot] (sum) {$\sum_j$};
		\node[block0,  right=.6  of sum] (fc22) {$h_{\text{rel}}$};
		\node[block0,  right=.6  of fc22] (fc23) {$h_{\text{lin}}$};
	
		\node[blockcircle,  right=.6  of fc23] (sum2) {$+$};		
		\node[block1,  right=.6  of sum2] (out) {$\bm s_i^{l+1}$};
 		\draw[line] (fc21.east)-| node [below, pos=0.25] {} node [above, pos=0.25] {}  (dot);
		\draw[line] (x.east)-| node [below, pos=0.25] {} node [above, pos=0.25] {}  (sum2);
	\draw[line] (x)-- (fc11);
		\draw[line] (z)-- (fc21);
			\draw[line] (fc11)-- (dot);
				\draw[line] (dot)-- (sum);
					\draw[line] (sum)-- (fc22);
						\draw[line] (fc22)-- (fc23);
						\draw[line] (fc22)-- (fc23);
						\draw[line] (fc23)-- (sum2);
						\draw[line] (fc23)-- (sum2);
						\draw[line] (sum2)-- (out);
					\node[draw=blue, thick,inner xsep=3mm,inner ysep=3mm,fit=(fc11)(fc21)(sum2),label={[label distance=-3.9cm,text depth=0,rotate=0]125: Vertex Update  ($\bm{\mathcal{V}}$)}]{}; 		
		\end{tikzpicture}	
			\vspace{.5cm}
		\label{fig:architecture_0}
	\end{subfigure}

\begin{subfigure}{1\textwidth}
	\centering
	\begin{tikzpicture}
	\node[block1] (zij) {$\bm z_{ij}^{l-1}$}; 
\node[block0, right=.6 of zij] (concat) {concat}; 
\node[block1, above  =.6  of concat] (x) {$\bm s_i^l, \bm s_j^l$};
\node[block0, right  =.6  of concat] (fc) {$h_\text{rel}$};
\node[block1, right  =.6  of fc] (zout) {$\bm z_{ij}^l$};
\draw[line] (zij)-- (concat);
\draw[line] (x)-- (concat);
\draw[line] (x)-- (concat);
\draw[line] (concat)-- (fc);
\draw[line] (fc)-- (zout);
\node[draw=orange, thick, inner xsep=3mm,inner ysep=3mm,fit=(concat)(fc), label={[label distance=-2.3cm,text depth=0,rotate=0]125: Edge Update ($\bm{\mathcal{E}}$)}]{}; 
	\end{tikzpicture}	
	\vspace{1cm}
	\label{fig:architecture_b}
\end{subfigure}	

	\begin{subfigure}{1\textwidth}
		\centering
		\begin{tikzpicture}
	\node[block1] (xin) {$\bm s_{i}$}; 
\node[block1, below=.6 of xin] (zin) {$\bm z_{ij}$}; 
\node[block0, right=.6 of zin] (rbf) {rbe}; 
\node[blockE, right=.6 of rbf] (E1) {$\bm{\mathcal{E}}$}; 
\node[blockE, right=.6 of E1] (E2) {$\bm{\mathcal{E}}$}; 
\node[block0, right=1.4 of xin] (ebd) {ebd}; 
\node[blockS, right=.6 of ebd] (S1) {$\bm{\mathcal{V}}$}; 
\node[blockS, right=.6 of S1] (S2) {$\bm{\mathcal{V}}$}; 
\node[block0, right=.6 of S2] (fc1) {$h_{\text{rel}}$}; 
\node[block0, right=.6 of fc1] (fc2) {$h_{\text{lin}}$}; 
\node[block1, below=.6 of fc2] (y) { $\pi^i$}; 	\draw[line] (zin)-- (rbf);
\draw[line] (rbf)-- (E1);
\draw[line] (E1)-- (E2);
\draw[line] (xin)-- (ebd);
\draw[line] (ebd)-- (S1);
\draw[line] (S1)-- (S2);
\draw[line] (S2)-- (fc1);
\draw[line] (fc1)-- (fc2);
	\draw[line] (ebd.east)-| node [below, pos=0.25] {} node [above, pos=0.25] {}  (E1);
		\draw[line] (E1.east)-| node [below, pos=0.25] {} node [above, pos=0.25] {}  (S1);
			\draw[line] (S1.east)-| node [below, pos=0.25] {} node [above, pos=0.25] {}  (E2);
				\draw[line] (E2.east)-| node [below, pos=0.25] {} node [above, pos=0.25] {}  (S2);
				\draw[line] (fc2)-- (y);
		\end{tikzpicture}	
		\label{fig:architecture_c}
	\end{subfigure}	
	\caption{The proposed neural message-passing (NMP) architecture. The full architecture consists of recurring vertex update (V) blocks and edge update (E) blocks. The smaller blocks (fc, rl, concat, ebd, rbe, $\cdot$, and +) are described in Section \ref{sec:architecture}.}
	\label{fig:architecture}
\end{figure}

\noindent The input features for the network can either be the raw properties of each vertex or their spatial positions. In this architecture, the interaction block is crucial to capturing the relative correlation of adjacent vertices. The input states, or observations, of all agents are denoted as $\bm s= (\bm s_i )_{i=1}^{N}$. We define $h_\text{lin}$ and $h_\text{rel}$ as fully connected neural network layers with linear and ReLU \cite{relu} activation functions, respectively. These layers are the backbone for all message-passing blocks in the framework. The weight matrix $\bm W$ of this layer provides a linear combination of input features, biased by $\bm b$ as $$h_\text{lin}(\bm x) =  \bm W\bm x  + \bm b.$$

\noindent In $h_\text{rel}$, a ReLU activation function \cite{relu} is applied to yield
$$h_\text{rel}(\bm x) = \max(0, \bm W \bm x + \bm b).$$

\noindent The summation block $\sum_j$ indicates the sum of all edges $\bm z_{ij}$ in the neighbourhood of $i$.\\

\noindent \textbf{Radial Basis Expansion (rbe):} Since one-hot encoding is the preferred representation for categorical features with neural networks, we argue that a similar formulation is also perferable for the scalar distance. Hence, the distance between two agents $d_{ij}$ is encoded as an $n_{\text{max}}$-dimensional feature vector $\bm z_{ij}$ using a radial basis function, with the $n$-th element as
\begin{equation}
z_{ij, n} =  \exp\bigg( -\frac{(d_{ij}- n\Delta d)^2}{\Delta d} \bigg),
\end{equation}
where $\Delta d$ is the increment step size. Depending on the RL environment setting, we set $n\Delta d$ equal to the maximum view range of all agents, and $n_{\text{max}}=10$ is the chosen dimension.\\

\noindent \textbf{Training Setting:} The Adam optimizer \cite{adam} with an initial learning rate of $0.01$ is implemented. A reduce-on-plateau schedule is used for decaying the learning rate with respect to training progression: the learning rate is decreased by $5\%$ for every $10$ consecutive batches with no reward improvements. \\

\section{Experiments and Results}\label{sec:result}

To verify our theories and assess Q-MARL's performance in situ, we conducted experiments in typical MARL scenarios and compared the results to those in the literature.

\subsection{Scenarios}\label{sec:scenarios}
\noindent The three grid-world MARL scenarios we conducted were commonly-used: \textit{jungle}, \textit{battle}, and \textit{deception} \cite{Lowe17}. In all three scenarios, each agent occupies one cell in a grid-world environment. Overlapping occupants are allowed in any cell. At any time step, agents can apply their actions to the environment and change the relative properties accordingly. The action of an agent are governed by its learned policy, with input states as its local observation, which is a $3\times 3$ rectangle surrounding the cell it occupies. Homogeneous agents share the same policy. Two agents are directly connected in the graph if one is present in the local observation of the other. The action space is common for all three scenarios. At each time step, an agent has 5 actions to choose from: move up, move down, move left, move right, and stay idle.  Actions are only eligible if they do not result in the agent occupying an already-occupied cell, say by a wall or other obstacle. An illustration of the different types of physical interactions between agents and with the environment is provided in Figure \ref{fig:scenarios}. The different rules, goals, and reward schemes assiciated with each of the three scenarios follows. The interactions between agents and rewarding schemes vary across the 3 scenarios are described as follows.\\

\begin{figure}
	\includegraphics[width=\linewidth]{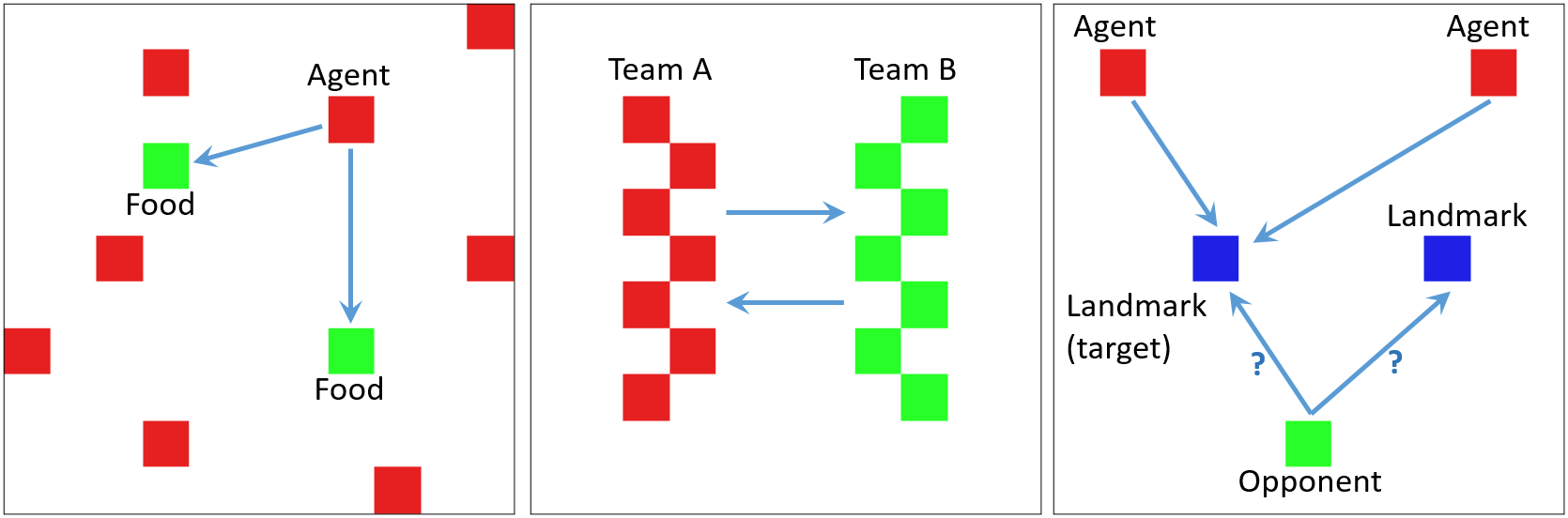}
	\caption{Illustrations of the three MARL scenarios considered in this paper. From left to right: Jungle, Battle, Deception.}
	\label{fig:scenarios}
\end{figure}

\noindent \textbf{Jungle:} In this scenario, agents must search for stationary foods placed in the environment, and they are allowed to attack and kill each other in the process. The objective is to maximize the long-term total amount of food consumed (reward) for all agents. This is ostensibly a collaborative scenario that poses a difficult social-dilemma, where agents must find ways to obtain food without killing each other. The short-term rewards for killing are attractive but detrimental to the long-term rewards. We formulated the specific environment and reward scheme as follows. In each time step, an agent receives an instantaneous reward $+1$ for sitting next to a food cell, and $0$ otherwise. If an agent is adjacent to at least one other agent three times, then it is killed and removed from the environment. The Food cells, however, permanently exist.\\

\noindent \textbf{Battle:} In this scenario, two teams, each with $N$ agents, are placed randomly in the environment. The goal is to have the most agents alive at the end of the episode, pre-defined as a number of time steps. Each agent from a particular team can be killed if surrounded by at least 3 agents from the opposing team. Since this is a collaborative-competitive game, where winning or losing is the ultimate goal, no instantaneous reward is assigned to any agent in any time step. Instead, the reward follows a simple positive scalar of $+1$ when a team wins, and $-1$ when it loses. The main challenge to winning an episode is developing an effective collaboration strategy for connected agents.\\

\noindent \textbf{Deception:} In this scenario, $N$ home team agents and $1$ adversary are randomly placed in the environment, along with a few stationary landmarks, one of which is the target that both teams are trying to reach. The home agents know which landmark is the target, while the adversary does not. No attacking or killing is allowed in this scenario. At the end of an episode, the home team is rewarded $+1$ if both the following requirements hold: the adversary agent has not found the target landmark and there is at least one home agent occupying the target landmark, and $-1$ otherwise. The adversary is co-trained with the home team to achieve its maximized reward of correctly occupying the target landmark (which it does not know). The major challenge in this scenario is that home agents tend to be attracted to the target landmark, therefore, the adversary agent can track the behaviour of the home team to identify that target, resulting in a low reward for the home team. \\






\subsection{Results}\label{sec:mainresult}
\noindent In the Jungle scenario, Q-MARL captured crucial information about the position of food cells and the surrounding situation. Some examples of the typical policies trained are depicted in Figure \ref{fig:animation}. (Note that these plots are from the perspective of one agent simply for demonstration purposes). Overall, the strategy for a team of eight agents was to fix the positions of some agents near known food cells, while others continued searching for more. As the example in the left panel of Figure \ref{fig:jungle-animation} shows, even though the food is not present in Agent 1's local view (the light blue rectangle), the recommended actions (marked as black arrows) encourage it to move away and look for additional food cells. This behaviour comes as a result of aggregating the local views of other agents who have already partially surrounded that food. Agent 2 is encouraged to move toward the food so as to block the remaining open cell, while Agent 3 is also encouraged to continue searching in new directions. The middle and right sub-figures illustrate the next two time steps of the scenario.\\

\begin{figure}[]
	\begin{subfigure}{.9\textwidth}
		\centering
		\includegraphics[width=1\linewidth]{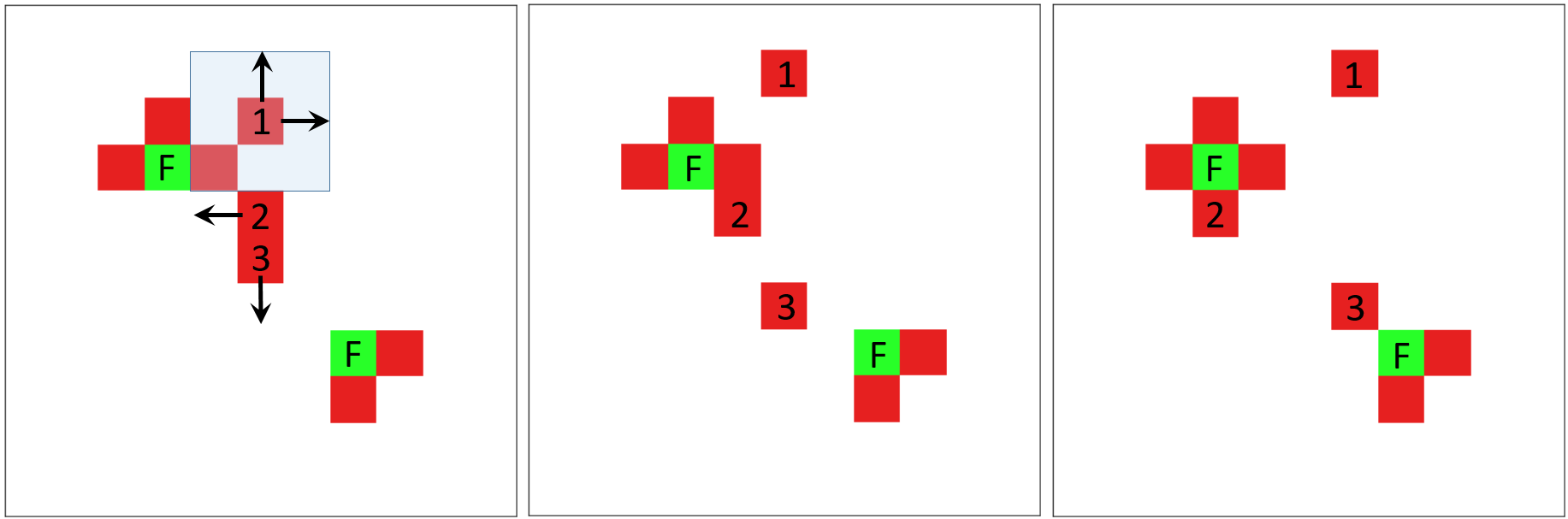}  
		\caption{Jungle}
		\label{fig:jungle-animation}
	\end{subfigure}
	\begin{subfigure}{.9\textwidth}
		\centering
		\includegraphics[width=1\linewidth]{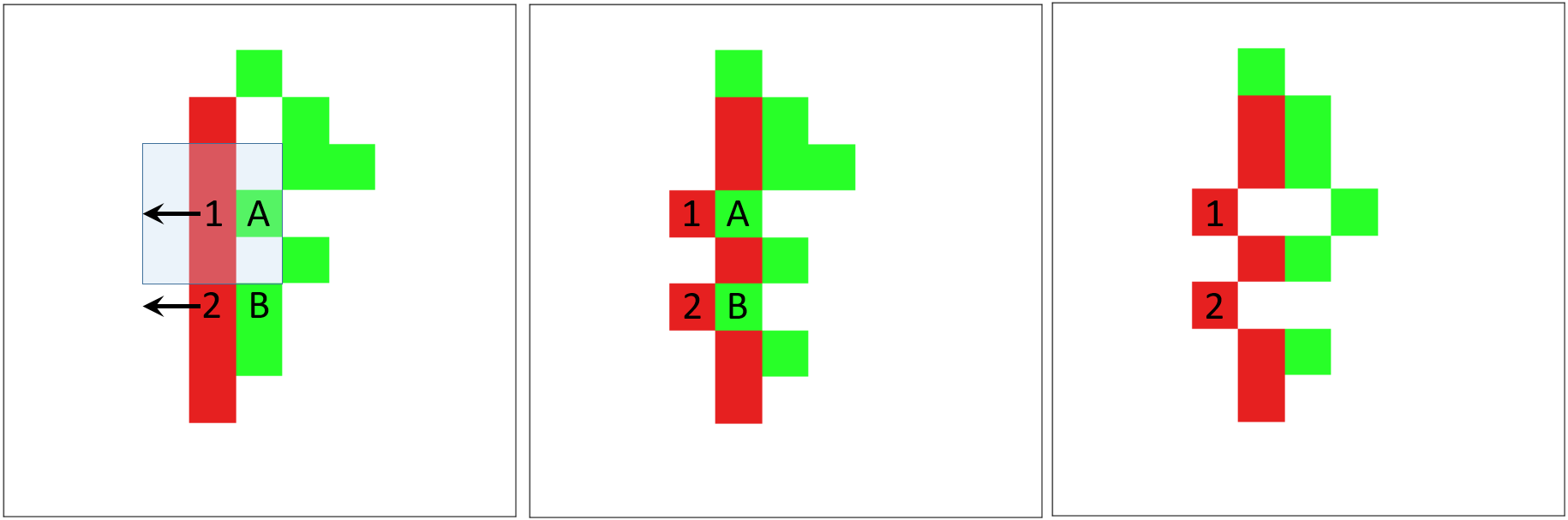}
		\caption{Battle}
		\label{fig:battle-animation}
	\end{subfigure}
	\begin{subfigure}{.9\textwidth}
		\centering
		\includegraphics[width=1\linewidth]{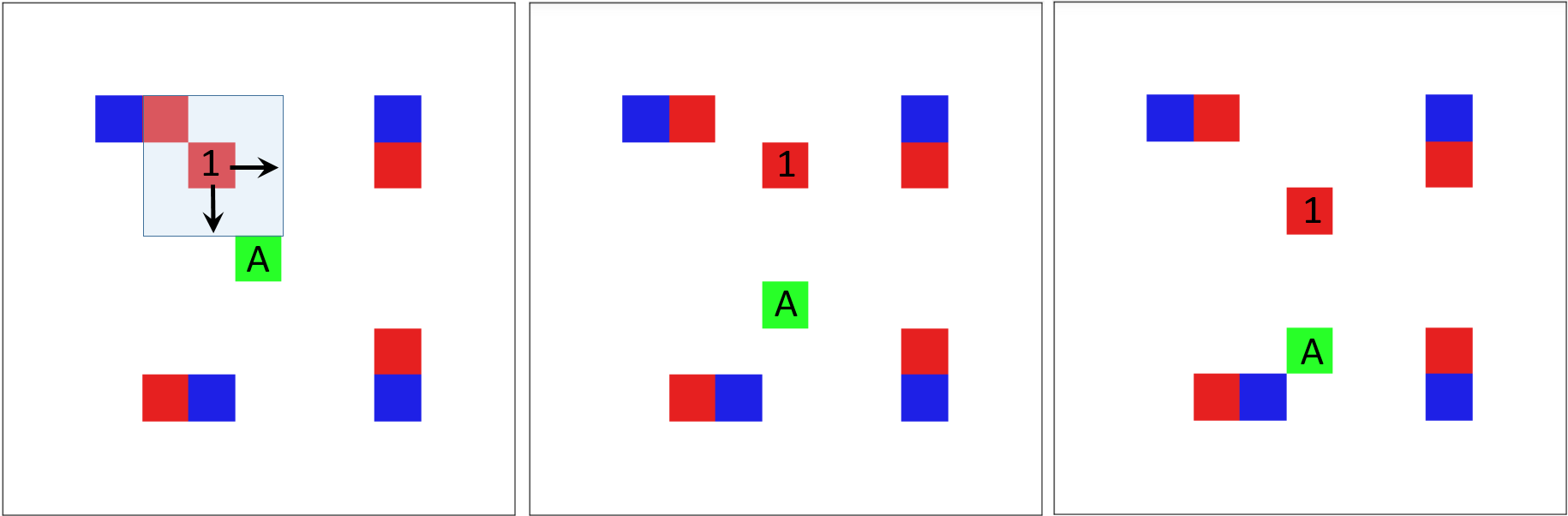}
		\caption{Deception}
		\label{fig:deception-animation}
	\end{subfigure}
	\caption{Illustration of typical trained behaviour in the three MARL scenarios, described in detail in the text of Section \ref{sec:result}.}
	\label{fig:animation}
\end{figure}

\noindent In the Battle scenario (Figure \ref{fig:battle-animation}), the home team's agents (red) attempt to maintain a straight line formation to avoid being surrounded and killed. Giving the agents flexibility to explore different types of attack and defence formations meant they could find a smart strategy for killing their opponents. The middle panel shows, Agents 1 and 2 moving toward the left so as to lure opposing Agents A and B into a trap and kill them. Again, this strategy is the result of considering the graph information of all neighbouring agents.\\

\noindent The strategy the home team developed in the Deception scenario was to try and occupy all landmarks so as to deceive the adversary (A). This behaviour is an improved version of the naive tactic that optimizes the short-term reward by ignoring non-target landmarks. Conventionally, the adversary can develop its own scavenging strategies to search for a landmark. The probability of a landmark that is not covered by home agents being a non-target is high in the naive training case. Therefore, the adversary is discouraged from occupying that landmark and continues to search for others. Since the home agents and adversary agent are trained simultaneously, any changes in the home team's policy lead to an improved adversary policy. This dilemma of co-evolving adversarial policies is effectively resolved by Q-MARL. Interestingly, we observed a special characteristic in the Deception scenario that was not present in Jungle and Battle: the smaller disconnected local sub-graphs were prominent during training. Each sub-graph represents a cluster of agents at one landmark. As a result, by means of NMP, all agents that discovered a specific landmark by chance, regardless of whether it was a target or not, were encouraged to occupy the landmark. This improved behaviour camouflaged the target landmark, because there was nothing distinctive for the adversary to discover.\\

\begin{figure}[]
	\begin{subfigure}{.3\textwidth}
		\centering
		\includegraphics[width=1\linewidth]{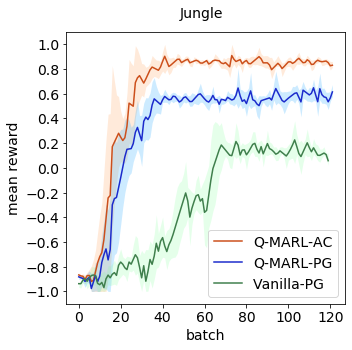}  
		\caption{Jungle}
		\label{fig:sub-first}
	\end{subfigure}
	\begin{subfigure}{.3\textwidth}
		\centering
		\includegraphics[width=1\linewidth]{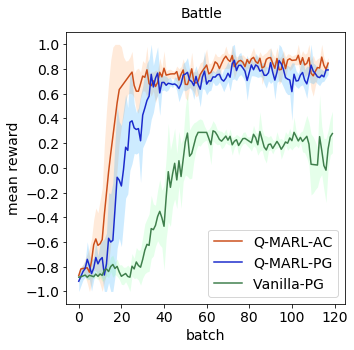}
		\caption{Battle}
		\label{fig:sub-second}
	\end{subfigure}
	\begin{subfigure}{.3\textwidth}
		\centering
		\includegraphics[width=1\linewidth]{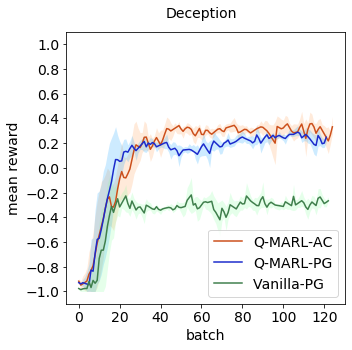}
		\caption{Deception}
		\label{fig:sub-third}
	\end{subfigure}
	\caption{Performance comparison of the mean training reward between Vanilla-PG with Q-MARL-AC and Q-MARL-PG.}
	\label{fig:result}
\end{figure}

\noindent The results for the different variants of Q-MARL compared to vanilla-PG is depicted in Figure \ref{fig:result}. In each sub-figure, the $x$-axis is the number of training batches, each consisting of 100 games, and the $y$-axis is the mean reward of the batch. In all games, there were distinct differences in performance between three algorithms. \\

\noindent In Jungle, Q-MARL-AC won 91.16\% of the games. Food cells are randomly placed in the environment for each game; hence, this scenario is able to let the algorithm stochastically improve the learning process, as the agents must develop their own strategies to search for food using only their local views, i.e., follow four walls clockwise until meeting another agent or meeting a food cell. This scenario therefore witnessed the peak performance for Q-MARL-AC ($90.94\%$), compared to Q-MARL-PG ($64.86\%$) and vanilla-PG ($22.66\%$). In Battle, Q-MARL variants outperformed vanilla-PG ($30.00\%$), although the gap between Q-MARL-PG ($87.13\%$) and Q-MARL-AC ($90.16\%$) was not as large as that in Jungle. Similarly, in Deception, Q-MARL-AC ($35.90\%$) performed slightly better than Q-MARL-PG ($28.91\%$), and both greatly surpassed vanilla-PG ($-21.87\%$).

\subsection{Comparison to Comtemporary Methods}\label{sec:compare}
\noindent In this section, we compare our results to observations in the literature from two perspectives: the theoretical strengths and weaknesses of the approaches' capability for problem solving, and the empirical performance numbers. Three recent works on graph-based RL make for suitable comparison as follows.\\

\begin{figure}[]
	\begin{subfigure}{1\textwidth}
		\centering
		\includegraphics[width=.95\linewidth]{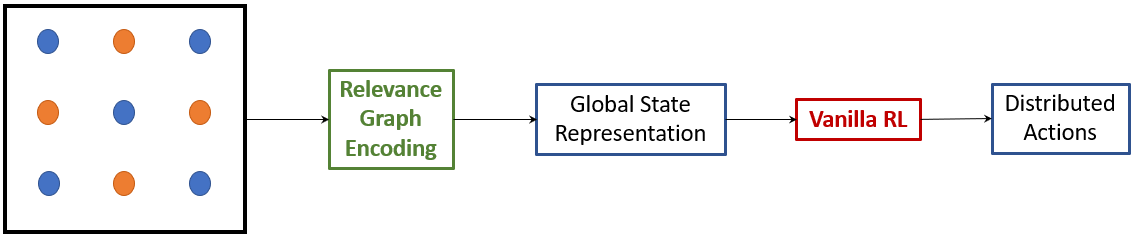}  
		\caption{Relevance Graph Embedding \cite{Malysheva18}}
		\label{fig:Malysheva}
		\hspace{2cm}
	\end{subfigure}
	
	\begin{subfigure}{1\textwidth}
		\centering
		\includegraphics[width=1\linewidth]{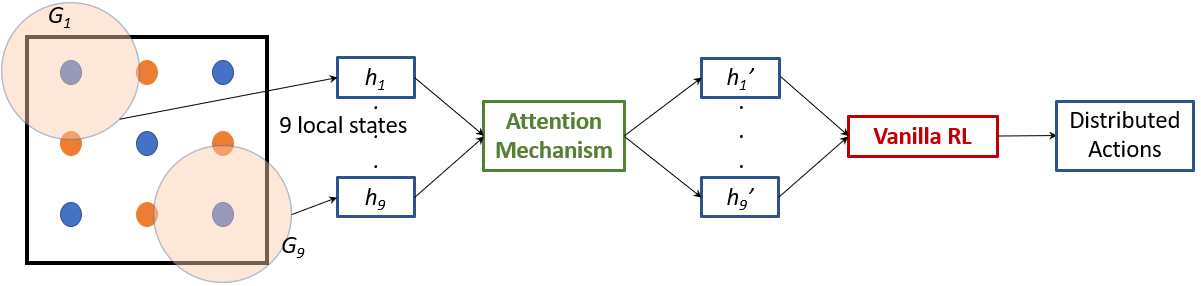}
		\caption{Single-head Attention and Multi-head Attention \cite{Agarwal19, Jiang20}}
		\label{fig:Agarwal}
		\hspace{2cm}
	\end{subfigure}
	\begin{subfigure}{1\textwidth}
		\centering
		\includegraphics[width=1\linewidth]{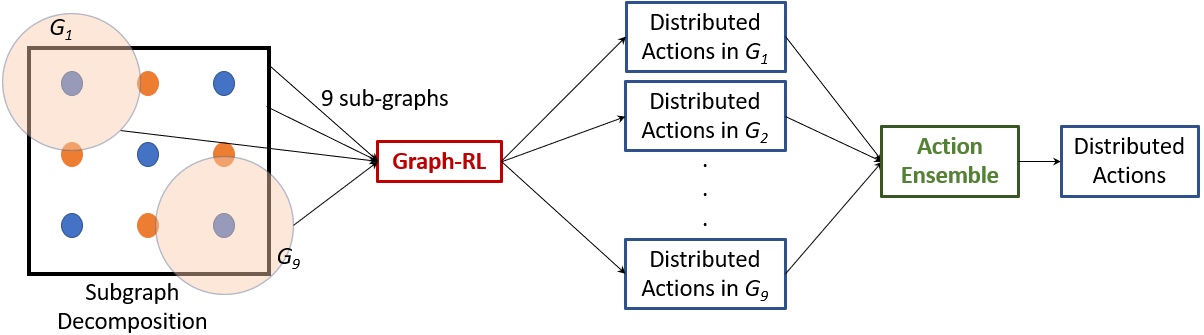}
		\caption{Our Work}
		\label{fig:Kha}
	\end{subfigure}
	\caption{High-level summary of related state-of-the-art studies in graph-based MARL}
	\label{fig:method_diagrams}
\end{figure}

\noindent $\bullet$ Relevance Graph Encoding (RGE) by Malysheva \textit{et al.} (2018) \cite{Malysheva18}: RGE produces \textit{relevance graphs} that describes the relationship between agents and environment
objects, given the relevance graph of the previous time step, and recent states/actions from all agents. At each specific time step, a new graph that represents the whole environment (with all agents) is retrieved and fed into the RL algorithm. There are two major disadvantages of this approach. First, because it uses information about all the agents in each time step, it does not scale larger envionments. In fact, the largest experimental scenario includes four agents. Second, it uses previous states and actions for learning, which violates the fundamental Markov property required for RL. A high-level overview of this approach is provided in Figure \ref{fig:Malysheva}. The distinctive contrast between our work and Malysheva's is that our method decomposes the full graph into sub-graphs and distributes actions to individuals as deliverables. This helps to generalise Q-MARL to larger environments and significantly reduces training burden. \\

\noindent $\bullet$ Single/Multi-head Attention by Jiang \textit{et al.} (2020, \cite{Jiang20}), Agarwal (2019, \cite{Agarwal19}): These two works rely on communications between agents in a local neighbourhood, as does Q-MARL. However, these works require agent indexing and a full adjacency matrix for the subsequent attention mechanism. Each agent extracts its own local state ($h_i$) then learns its dependency with neighbouring states via a convolutional network to produce \textit{better} individual states $h_i'$. The convolved states ($h_i'$) are then fed into a vanilla RL algorithm. Agarwal \cite{Agarwal19} proposed single-head attention. Jiang \cite{Jiang20} subsequently improved on this with a multi-head mechanism that simply concatenates several attentional heads before feeding them into a non-linearity function, based on a shallow multi-layer-perceptron (MLP). \\

\noindent The biggest difference between all these strategies and our approach, is the principle design for how neighbourhood information is learned. Instead of extracting local representative states for individual agents and improving upon them through an attention mechanism, we feed raw features from the perspective of any given agent into the graph network. Our method therefore offers more capacity for learning and, hence, better generalisation since it can yield more subtle dependencies across neighbouring agents. Further, each output sample from the graph network is a set of actions from all agents involved in a sub-graph. As the same agent may present in multiple sub-graphs, this leaves room for a subsequent action ensembling step for more robust action decisions. High-level summary diagrams of these studies as well as our work are provided in Figure \ref{fig:method_diagrams}.\\

\begin{table}[t]
	\centering
	\caption{Generalisation performance in the three scenarios Jungle, Battle, and Deception. Each method was trained with the number of agents indicated at the top of the left column. The trained policy was then applied to a larger scenario, as shown in the right column.}
	\begin{tabu}{Dm{2.3cm}|Dm{1.3cm}Dm{1.3cm}|Dm{1.3cm}Dm{1.3cm}|Dm{1.3cm}Dm{1.3cm}}
		\multirow{2}{*}{Method}
		& \multicolumn{2}{c|}{Jungle}  &  \multicolumn{2}{c|}{Battle} 
		& \multicolumn{2}{c}{Deception}   \\ & $N=4$ & $N=8$ & $N=14$ & $N=28$ & $N=1$& $N=2$  \\
		
		\tabucline[1pt]{-}
		
		\multirow{3}{*}{RGE \cite{Malysheva18}}
		&$0.952$&$0.667$& $0.733$& $0.681$ &$0.252$ &$0.210$\\ &$0.943$&$0.638$& $0.752$& $0.689$ &$0.278$ &$0.207$ \\
		&$0.951$&$0.614$& $0.710$& $0.694$ &$0.246$ &$0.222$  \\
		\tabucline[1pt]{-}
		
		\multirow{3}{*}{SHA \cite{Agarwal19}}
		&$0.861$&$0.746$& $0.817$& $0.699$ &$-0.029$ &$-0.294$\\ &$0.884$&$0.779$& $0.815$& $0.726$ &$-0.185$ &$-0.284$ \\
		&$0.893$&$0.753$& $0.818$& $0.708$ &$0.006$ &$-0.188$  \\
		\tabucline[1pt]{-}
		
		\multirow{3}{*}{MHA \cite{Jiang20}}
		&$0.905$&$0.780$& $0.894$& $0.716$ &$0.107$ &$-0.113$\\ &$0.897$&$0.779$& $0.870$& $0.704$ &$0.091$ &$-0.091$ \\
		&$0.897$&$0.777$& $0.885$& $0.729$ &$0.152$ &$-0.087$ \\
		\tabucline[1pt]{-}
		
		\multirow{3}{*}{Q-MARL}
		&$0.925$&$0.911$& $0.955$& $0.848$ &$-0.163$ &$-0.196$\\ &$0.930$&$0.917$& $0.960$& $0.863$ &$-0.144$ &$-0.164$\\
		&$0.922$&$0.926$& $0.960$& $0.849$ &$-0.127$ &$-0.155$ \\
		\tabucline[1pt]{-}
		
	\end{tabu}
	\label{table:gridsearch}
\end{table}

\begin{figure}[]
	\begin{subfigure}{1\textwidth}
		\centering
		\includegraphics[width=1\linewidth]{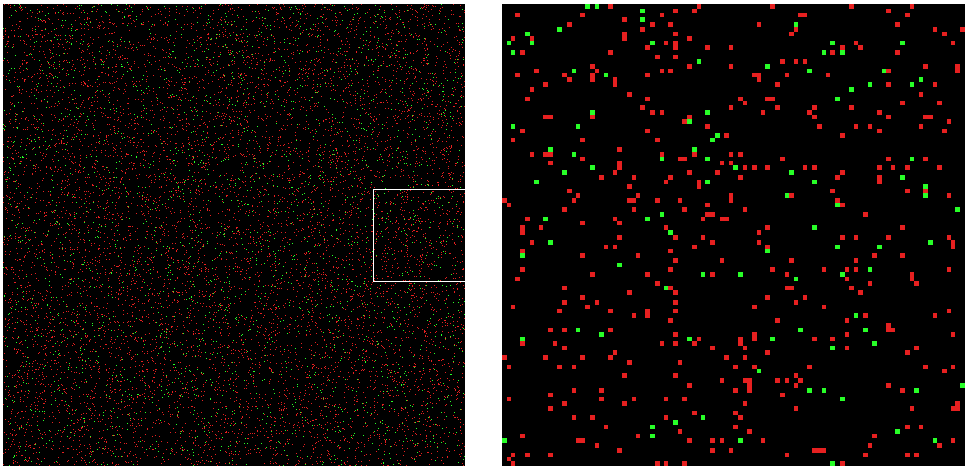}
		\caption{Full initial environment (left) and the zoomed region (right) corresponding to the white square box in the left figure.}
		\label{fig:jungle_init_big}
		\hspace{2cm}
	\end{subfigure}
	\begin{subfigure}{1\textwidth}
		\centering
		\includegraphics[width=1\linewidth]{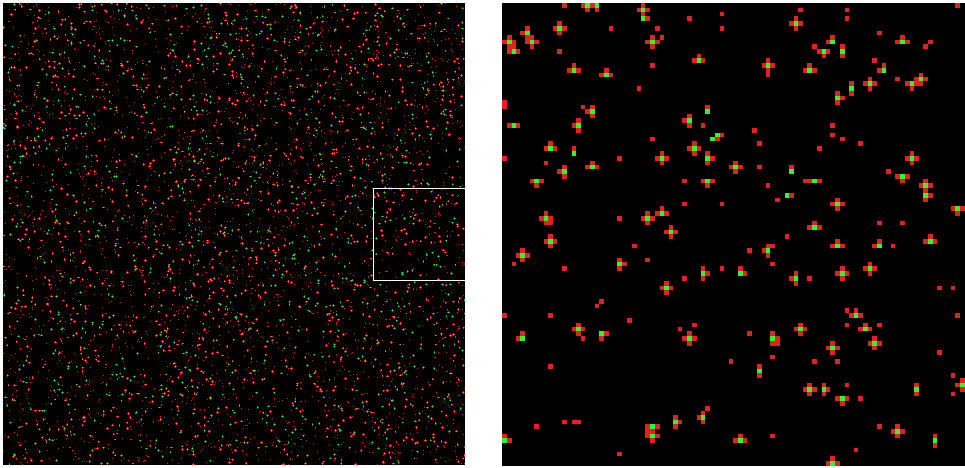}
		\caption{Full environment after 1000 time steps performed by a trained model (left) and the zoomed region (right) corresponding to the white square box in the left figure.}
		\label{fig:jungle_trained_big}
	\end{subfigure}
	\caption{An illustration of a big scale Jungle scenario with 10,000 agents (red dots) and 2500 foods (green dots).}
	\label{fig:jungle_big}
\end{figure}

\noindent In terms of empirical performance, we assessed how well each algorithm generalised to a larger environment and the training speed. As shown in Table 1, each scenario was trained with a small number of agents, then tested with double that number. Known as \textit{curriculum learning} \cite{bengio09}, this is an effective method for evaluating the generalisation of an algorithm.\\

\noindent Our observations of the results follow:

\begin{itemize}
\item It is evident with all methods that performance suffered when moving from the smaller training environment to the larger testing environment. 
\item In the Jungle scenario, Q-MARL showed the best generalisation ability, despite RGE received the highest training rewards. The explanation is that a full graph embedding better captures the global information than a series of sub-graphs - an approach that is only feasible in small scenarios. 
\item In the Battle scenarios, Q-MARL amassed the most rewards in both scenario sizes.
\item Q-MARL failed miserably at the Deception scenario, accumulating the least rewards of all during training and performing only marginally better than SHA in testing. We attribute its inadequacy to the underlying nature of the Deception task, where separate sub-graphs are unable to spread agents to separate landmarks. 
\end{itemize}

\noindent To evaluate training speed, we tested all methods in each of the different scenarios with 10, 100, 1000, and 10000 agents. Figure \ref{fig:jungle_big} shows an example from the Jungle scenario and Figure \ref{fig:perf} shows the time (in seconds) required to update 100 episodes for each setting. From the result, we find that at $N>100$ agents, the computational burden exploded for RGE, while all graph-based approaches (SHA, MHA, Q-MARL) remained feasible. Strikingly, as the number of agents increased, Q-MARL performed better and better as compared to SHA and MHA. This is because the graph decomposition procedure prunes substantially more unnecessary information about agents that are far away as the number of entities grows. 

\begin{figure}
	\includegraphics[width=.8\linewidth]{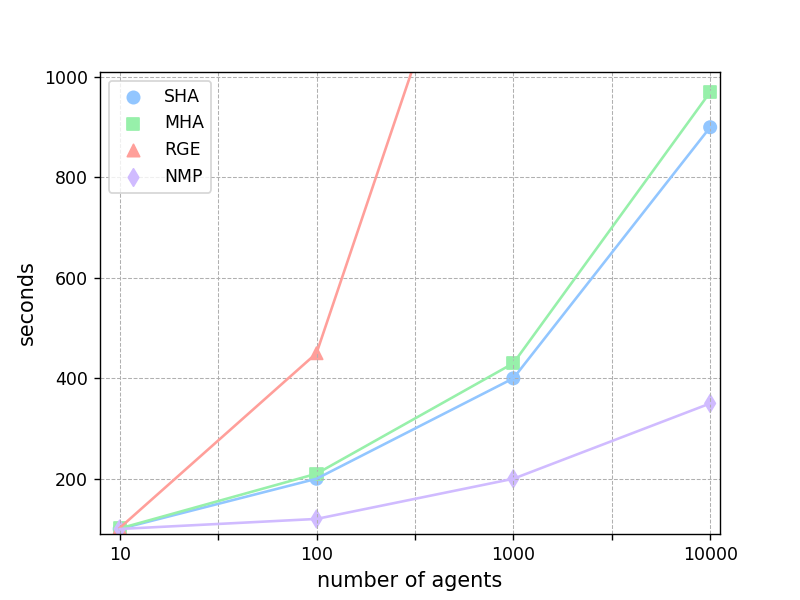}
	\caption{Training speed comparison. The horizontal axis represents the number of agents (10, 100, 1,000, and 10,000 agents). The vertical axis indicates the time (in seconds) required to update 100 episodes, averaged on three scenarios (Jungle, Battle, Deception).}
	\label{fig:perf}
\end{figure}

\section{Conclusion}\label{sec:conclusion}
\noindent The graph-based approach to MARL presented in this paper addresses the long-standing issue of scalability with this paradigm. At each time step, the full environment is decomposed into multiple sub-graphs, each consisting of a limited number of agents that form a local neighbourhood. Conventional RL algorithms, such as policy gradient \cite{pg} or actor-critic \cite{AC}, learn by using decomposed sub-graphs as training samples. The graph decomposition and architecture design were inspired by state-of-the-arts from quantum chemistry. Experiments with typical RL scenarios support our theoretical proof and illustrate that our graph-based approach can substantially reduces training time and training loss in terms of performance, and can also significantly better generalise when transitioning from a small curriculum scenario to a larger one. Q-MARL repository is available for download at https://github.com/cibciuts/NMP\_MARL.

\section*{Acknowledgement}
\noindent This work was supported in part by the Australian Research Council (ARC) under discovery grant DP220100803 and DP250103612 and ITRH grant IH240100016. Research was also sponsored in part by the US Office of Naval Research Global under Cooperative Agreement Number ONRG - NICOP - N62909-19-1-2058. We would also like to thank the anonymous reviewers for their constructive comments on this work.

\nocite{*}
\bibliographystyle{plain}
\bibliography{\jobname}

\end{document}